\documentclass{article}

\usepackage{natbib}
\usepackage{times}
\usepackage{mathtools}
\usepackage{microtype}

\usepackage{algorithm}
\usepackage{algorithmic}
\usepackage{graphics}
\usepackage{multirow}

\usepackage{paralist}
\usepackage{appendix,url}

\usepackage{subfigure}
\usepackage{color}
\usepackage{soul}
\usepackage{exscale,relsize}
\usepackage{amsmath,amsthm,amssymb}
\usepackage{verbatim}
\usepackage{mathrsfs}
\usepackage{color}
\usepackage{sidecap}

\newtheorem{thm}{Theorem}
\newtheorem{cor}[thm]{Corollary}

\newcommand{\R}{\mathbb{R}}
\newcommand{\la}{\left\langle}
\newcommand{\ra}{\right\rangle}
\let\bs=\boldsymbol

\let\set=\mathcal
\newcommand{\para}[1]{\noindent{\bf #1}}

\def \Diag {\mathrm{Diag}}
\def \diag {\mathrm{diag}}
\def \minimize {\textup{minimize} }
\def \maximize {\textup{maximize}}
\def \subjectto {\textup{subject to}}
\def \temp {\textup{temp}}
\def \gt {\textup{gt}}

\usepackage{hyperref}

\newtheorem{remark}{\textbf{Remark}}

\usepackage[accepted]{icml2014}

\icmltitlerunning{Semidefinite Relaxation for MAP Estimation}

\begin{document}

\twocolumn[
\icmltitle{Scalable Semidefinite Relaxation for Maximum {\em A Posterior} Estimation}

\icmlauthor{Qixing Huang}{huangqx@stanford.edu}
\icmladdress{Department of Computer Science, Stanford University, Stanford, CA 94305, USA }
\icmlauthor{Yuxin Chen}{yxchen@stanford.edu}
\icmladdress{Department of Electrical Engineering, Stanford University, Stanford, CA 94305, USA}
\icmlauthor{Leonidas Guibas}{guibas@cs.stanford.edu}
\icmladdress{Department of Computer Science, Stanford University,
            Stanford, CA 94305 USA}

% You may provide any keywords that you
% find helpful for describing your paper; these are used to populate

\icmlkeywords{MAP, Semidefinite Programming, Convex Relaxation, Low-rank, Sparsity, ADMM}
\vskip 0.3in
]

\begin{abstract}
Maximum {\em a posteriori} (MAP) inference over discrete Markov random fields is a fundamental task spanning a wide spectrum of real-world applications, which is known to be NP-hard for general graphs. In this paper, we propose a novel semidefinite relaxation formulation (referred to as SDR) to estimate the MAP assignment. Algorithmically, we develop an accelerated variant of the alternating direction method of multipliers (referred to as SDPAD-LR) that can effectively exploit the special structure of the new relaxation. Encouragingly, the proposed procedure allows solving SDR for large-scale problems,  e.g., problems on a grid graph comprising hundreds of thousands of variables with multiple states per node. Compared with prior SDP solvers, SDPAD-LR is capable of attaining comparable accuracy while exhibiting remarkably improved scalability, in contrast to the commonly held belief that semidefinite relaxation can only been applied on small-scale MRF problems. We have evaluated the performance of SDR on various benchmark datasets including OPENGM2 and PIC in terms of boththe quality of the solutions and computation time. Experimental results demonstrate that for a broad class of problems, SDPAD-LR outperforms state-of-the-art algorithms in producing better MAP assignments in an efficient manner.

%However, due to the popularity and wide applicability of MAP formulations in various domains, deriving conditions under which the MAP estimation problem can be exactly and tractably solved is of great theoretical and practical importance. So far theoretical guarantees have been derived only for very specialized models like tree graphs or convex potential functions.

%In this paper, we introduce a family of data-dependent exact solution conditions by developing and analyzing a novel semidefinite relaxation (called SDR). Specifically, when the potential functions are given by noisy measurement among variables, we provide tight constant bounds on the level of noise for which MAP inference can be solved perfectly.

\end{abstract}

\section{Introduction}

Computing the maximum {\em a posteriori} (MAP) assignment in a graphical model is a central inference task spanning a wide scope of scenarios \cite{wainwright2008graphical}, ranging from traditional applications in graph matching, stereo reconstruction, object detection, error-correcting codes, gene mapping, etc., to a more recent application in estimating consistent object orientations from noisy pairwise measurements~\cite{crandall2011discrete}. For general graphs, this problem is well-known to be NP-hard~\cite{Shimony:1994:FMB}. However, due in part to its importance in practice,  a large body of algorithms have been proposed to approximate MAP estimates by solving various convex relaxation formulations.

%Despite the practical success of these algorithms, theoretical understanding --- under what conditions they lead to global optimal solutions --- has received little attention. It is known that the MAP estimation problem can be solved exactly for tree graphs or when the potential functions are submodular. However, these are specialized settings that rarely apply in practical problems. This motivates us to develop MAP estimation algorithms that admit more general exact inference conditions and which can lead to simple and powerful theoretical predictions when adapted to a large class of application scenarios.

Among those methods based on convex surrogates, semidefinite relaxation usually strictly dominates other formulations based on linear programming or quadratic programming in terms of solution quality. Despite its superiority in obtaining more accurate estimates, however, the most significant challenge that limits the applicability of any semidefinite relaxation paradigm on real problems is efficiency.
%Most semidefinite optimization procedures involve performing a full matrix eigen-decomposition at each iteration, which is computationally expensive.
So far existing general-purpose SDP solvers can only handle problems with small dimensionality.

In this paper, we propose a novel semidefinite relaxation approach (referred to as SDR) for second-order MAP inference in pairwise undirected graphical models. Our key observation is that the marginalization constraints in a typical linear programming relaxation (c.f.\cite{kumar2009analysis}) can be subsumed by combing a semidefinite conic constraint with a small set of linear constraints. As a result, SDR admits a concise set of nicely decoupled constraints, which allows us to develop an accelerated variant (referred as SDPAD-LR) of the alternating direction method of multipliers method (ADMM) that is scalable to very large-scale problems.

%In particular, our algorithm, referred to as SDPAD-LR, is able to solve SDR with accuracy comparable to state-of-the-art generic SDP solvers (e.g. SDPAD, MOSEK), while enjoying remarkably improved efficiency and scalability.
On a standard PC, we have successfully applied SDR on dense problems of dimensions of ($\#\textup{states}\times \# \textup{variables}$) up to five thousand, and on grid-structured problems up to $10^5$ variables each with dozens of states per node.

Practically, SDPAD-LR performs remarkably well on a variety of problems. We have evaluated SDPAD-LR on two collections of benchmark datasets: OPENGM2~\cite{Kappes:2013:OPENGM2} and a probabilistic inference challenge \cite{PIC:2011}. Each benchmark consists of multiple categories of problems derived from various MAP estimation tasks. Experimental results demonstrate that SDPAD-LR outperforms the state-of-the-art algorithms in computational speed, while often obtaining better MAP estimates.

\subsection{Background}

There is a vast literature concerning MAP estimation over discrete undirected graphical models and it is beyond the scope of this paper to discuss all existing algorithms. Interested readers are referred to \cite{wainwright2008graphical} for an in-depth introduction to this topic. In the following, we focus on methods that involve convex relaxation, which are the most relevant to our approach.

Many prior convex relaxation techniques are derived from the original graph structure underlying the MAP estimation problem, among which linear programming relaxation (LPR) methods~\cite{journals/siamdm/ChekuriKNZ04,Wainwright:2005:MAP} are the most popular. In addition to LPR, researchers have considered alternative convex relaxations, e.g., quadratic relaxation (QP-RL)~\cite{ravikumar2010message} and second-order cone relaxation (SOCP-MS)~\cite{kumar2009analysis}. In the seminal work of~\cite{kumar2009analysis}, the authors evaluate various convex relaxation approaches, and assert that LPR dominates QP-RL and SOCP-MS. However, as will be shown later, LPR is further dominated by a standard SDP relaxation~\cite{wainwright2008graphical}, which is one of the main foci of this paper.
%\cite{,globerson2007fixing,sontag2012tightening}

A recent line of approaches have aimed at obtaining tighter convex relaxations by incrementally adding higher-order interactions to enforce proper marginalization over groups of variables~\cite{sontag2012tightening,conf/eccv/KomodakisP08,journals/jmlr/BatraNK11}. Despite the practical success of these approaches, it remains an open problem to analyze their behavior --- for example, to decide whether a polynomial number of clusters are sufficient.

There have been several attempts in applying semidefinite relaxation to obtain MAP assignment~\cite{Torr:2003:SDP,olsson-eriksson-etal-cvpr-07,Wang:2013:FSA,PengHSX12}. However, most of these methods are primarily designed for binary MAP estimation problems. In a recent work, \cite{PengHSX12} considered a general MAP estimation problem, where each variable has multiple states. The key difference between the proposed formulation and that of ~\cite{PengHSX12} is that we utilize the semidefinite cone constraint to prune redundant linear marginalization constraints. This leads to a concise set of loosely decoupled constraints, which is important in developing effective optimization paradigms.
%In addition, we also provide an interesting analysis of the proposed semidefinite relaxation.

%The analysis of the proposed SDP relaxation is motivated from recent advances in low-rank matrix recovery, which root in analyzing the KKT optimality conditions of convex relations. However, due to the special structure of the SDP relaxation, the induced KKT optimality conditions in our setting turn out to be very different. This leads to novel techniques to construct the dual certificates and to derive the exact solution conditions.

%\subsection{Organization and Notation}
\subsection{Notation}

%The remainder of the paper is organized as follows. In Section~\ref{Sec:Formulation}, we introduce the proposed semidefinite relaxation. We then present how to efficiently solve the semidefinite program in Section~\ref{Sec:Optimization}. In Section~\ref{Sec:Results}, we evaluate the proposed approach on benchmark datasets and compare its performance against state-of-the-art techniques.  Section~\ref{Sec:Conclusions} concludes the paper with a summary of our findings.

Before proceeding, we introduce a few notations that will be used throughout the paper. For any linear operator $\set{A}$, we let $\set{A}^{\star}$ represent its conjugate operator. Denote by $\R^{N\times M}_{+}$ the set of $N\times M$ matrices with nonnegative entries, and $(\cdot)_{+}:\R^{N\times M}\rightarrow \R^{N\times M}_{+}$ the projection operator onto $\R^{N\times M}_{+}$. For any symmetric matrix $\bs{M}$, we use $\bs{M}_{\succeq \bs{0}}$ to represent the projection of $\bs{M}$ onto the positive semidefinite cone. Finally, we denote by $\|\boldsymbol{X}\|_\mathrm{F}$ the Frobenius norm of a matrix $\boldsymbol{X}$.

\section{MAP Estimation and SDP Relaxation}
\label{Sec:Formulation}

We start with state configurations over $n$ discrete random variables $\set{X} = \{x_1, \cdots, x_{n}\}$. Without loss of generality, assume that each $x_i$ takes values in a discrete state set $\{1,\cdots,m\}$. Consider a pairwise Markov random field (MRF) $\set{G}$ parameterized by the potentials (or sufficient statistics) $w_i(x_i)$ for all vertices and $w_{ij}(x_i,x_j)$ for all edges $(i,j)\in \set{G}$. The energy (or log-likelihood) associated with this MRF is given by
% Given a graph $\set{G} = (\set{V}, \set{E})$ with $n$ vertices, potentials for all variables $w_i(x_i) (1\leq i \leq n)$ and potentials $w_{ij}(x_i,x_j)$ for all edges $(i,j)\in \set{E}$, define the objective function
\begin{equation}
f(\set{X}) = \sum\limits_{i=1}^{n}w_i(x_i) + \sum\limits_{(i,j)\in \set{E}}w_{ij}(x_i, x_j).
\end{equation}
The goal of MAP estimation is then to compute the configuration of states that maximizes the energy -- the most probable state assignment $\set{X}_M$.

\subsection{Semidefinite Programming Relaxation (SDR)}

MAP estimation over discrete sets is an NP-hard combinatorial problem, and can be cast as an integer quadratic program (IQP). Denote by $\bs{x}_i = (x_{i,1}, \cdots, x_{i,m})^{\top}\in\{0,1\}^m$ a binary vector such that  $x_{i,j} = 1$ if and only if $x_i = j$. Then MAP estimation is equivalent to the following integer program.
\begin{align}
\textbf{(IQP):} \quad  \underset{\bs{x} \in \{0,1\}^{nm}}{\maximize} & \quad  \sum\limits_{i=1}^{n}\la \bs{w}_i, \bs{x}_i \ra + \sum\limits_{(i,j)\in \set{G}}\la \bs{W}_{ij}, \bs{x}_i \bs{x}_j^{\top} \ra \nonumber \\
\subjectto & \quad \bs{1}^{\top}\bs{x}_i = 1, \qquad 1\leq i \leq n, \label{Eq:Linear:Cons0}
\end{align}
where $\bs{w}_i$ and $\bs{W}_{ij}$ encode the corresponding potentials.

The hardness of the above IQP arises in two aspects: (i) $\bs{x}$ are binary-valued, and (ii) the objective function is a quadratic function of these binary variables. These motivate us to relax the constraints in some appropriate manner, leading to our semidefinite relaxation. In the sequel, we present the proposed relaxation in a  step-by-step fashion.

\begin{enumerate}[1)]
\item In the same spirit as existing convex formulations (e.g.,~\cite{kumar2009analysis,PengHSX12}), we introduce a binary block matrix $\bs{X}: = \bs{x}\bs{x}^{\top} \in \{0,1\}^{nm\times nm}$ to accommodate quadratic objective terms:
$$
\bs{X} = \left(
      \begin{array}{cccc}
        \Diag(\bs{x}_1) & \bs{X}_{12} & \cdots & \bs{X}_{1n} \\
        \bs{X}_{12}^{\top} & \Diag(\bs{x}_2) & \vdots & \vdots \\
        \vdots & \cdots & \ddots & \vdots \\
        \bs{X}_{1n}^{\top} & \cdots & \cdots & \Diag(\bs{x}_n) \\
      \end{array}
    \right),
$$
which apparently exhibits the following properties:
% each diagonal block obeys $\bs{X}_{ii} = \bs{x}_i\bs{x}_i^T = \Diag(\bs{x}_i)$ for all $1\leq i\leq n$.

%It is straightforward to see that each diagonal block obeys
\begin{equation}
\bs{X}_{ii} = \bs{x}_i\bs{x}_i^\top = \Diag(\bs{x}_i), \qquad 1\leq i \leq n.
\label{Eq:Diagonal:Cons}
\end{equation}

\item The non-convex constraint $\bs{X} = \bs{x}\bs{x}^{\top}$ is then relaxed and replaced by $\bs{X} \succeq \bs{x}\bs{x}^{\top}$, which by Schur complement condition is equivalent to the following semidefinite conic constraint :
\begin{equation}
\left(
  \begin{array}{cc}
    1 & \bs{x}^{\top} \\
    \bs{x} & \bs{X} \\
  \end{array}
\right) \succeq \bs{0}.
\label{Eq:SDP:Cons0}
\end{equation}

\item The binary constraints $\bs{x} \in \{0,1\}^{nm}$ and $\bs{X} \in \{0, 1\}^{nm \times nm}$ are replaced by weaker linear constraints 
    $$
    \bs{X} \geq \bs{0}.
    $$
    Note that the constraints $\bs{0} \leq \bs{x} \leq \bs{1}$ and $\bs{X}\leq \bs{1}\cdot\bs{1}^\top$ are essentially subsumed by the constraints (\ref{Eq:Linear:Cons0}), (\ref{Eq:Diagonal:Cons}), and~(\ref{Eq:SDP:Cons0}) taken together. For the sake of numerical efficiency, we further relax the non-negative constraint $\bs{X} \geq \bs{0}$ to be
\begin{equation}
\bs{X}_{ij} \geq \bs{0}, \quad (i,j)\in \set{G}.
\label{Eq:Block:NonNeg}
\end{equation}
As we will see later, this relaxation is crucial in accelerating SDP solvers for large-scale problems.
\end{enumerate}

\begin{remark}
The non-negativity constraints described in (\ref{Eq:Block:NonNeg}) are necessary since otherwise SDR becomes loose for submodular functions. Below is an example in the presence of 2 variables each having 2 states:
$$
\bs{w}_1 = \left[\begin{array}{c}
                   2 \\
                   0
                 \end{array}
\right], \quad \bs{w}_2 = \left[\begin{array}{c}
                                  -3 \\
                                  0
                                \end{array}
 \right], \quad \bs{W}_{12} = \left[\begin{array}{cc}
                                      0 & 2 \\
                                      2 & 0
                                    \end{array}
 \right].
$$
It is clear that $\bs{W}_{12}$ satisfies the submodular property. However, the optimizer of SDR after dropping the constraint $\bs{X}_{ij}\geq \bf{0}$  is given by
$$
\bs{x}_1 = \frac{1}{3}\left[\begin{array}{c}
                  1 \\
                  2
                 \end{array}
\right], \ \bs{x}_2 = \frac{1}{9}\left[\begin{array}{c}
                                  8 \\
                                  1
                                \end{array}
 \right], \ \bs{X}_{12} = \frac{1}{9}\left[\begin{array}{cc}
                                      4 & -1 \\
                                      4 & 2
                                    \end{array}
 \right],
$$ which does not obey the non-negativity constraint on $\bs{X}$.
\end{remark}

The feasibility constraints~(\ref{Eq:Linear:Cons0}),(\ref{Eq:Diagonal:Cons}), (\ref{Eq:SDP:Cons0}) and (\ref{Eq:Block:NonNeg}) taken collectively give rise to the following semidefinite relaxation (SDR) formulation for MAP estimation:
\begin{align}
\textbf{(SDR): } \text{ } \underset{\bs{x},\bs{X}}{\maximize}  & \quad \sum\limits_{i=1}^{n}\langle \bs{w}_i, \bs{x}_i \rangle + \sum\limits_{(i,j)\in \set{G}}\langle \bs{W}_{ij}, \bs{X}_{ij} \rangle \nonumber \\
\subjectto & \quad \left(
               \begin{array}{cc}
                 1 & \bs{x}^{\top} \\
                 \bs{x} & \bs{X} \\
               \end{array}
             \right) \succeq \bs{0},
\quad  \label{Eq:SDR:SDP} \\
                        & \quad \bs{X}_{ii} = \Diag(\bs{x}_i), \  1\leq i \leq n, \label{Eq:SDR:Linear1} \\
                        & \quad \bs{1}^{\top}\bs{x}_i  = \bs{1}, \qquad\ \ \  1\leq i \leq n, \label{Eq:SDR:Linear2} \\
                         & \quad \bs{X}_{ij} \geq \bs{0},  \qquad\quad \ (i,j) \in \set{G}.
\label{Eq:SDP:Cons}
\end{align}

\subsection{Comparison with Prior Relaxation Heuristics}

\subsubsection{\bf Superiority over LP relaxations.} Careful readers will remark that there might exist other convex constraints on $\bs{X}$ and $\bs{x}$ that we can enforce to tighten the proposed semidefinite relaxation. One alternative is the following marginalization constraints, which have been widely invoked in LP relaxation for MAP estimation:
\begin{equation}
\bs{X}_{ij}\bs{1} = \bs{1}, \quad  \bs{X}_{ij}^{\top}\bs{1} = \bs{1},\qquad 1\leq i < j \leq n.
\label{Eq:Linear:Cons}
\end{equation}
Somewhat unexpectedly, these constraints turn out to be redundant,
%They are inherently subsumed by the semidefinite constraint together with some linear constraints, 
as asserted in the following theorem.
\begin{thm}\label{thm_linear_redundancy} Any feasible solution $\boldsymbol{X}$ to SDR (i.e. any $\boldsymbol{X}$ obeying the feasibility constraints of SDR) necessarily satisfies
\begin{equation}
\bs{X}_{ij}\bs{1} = \bs{1}, \quad  \bs{X}_{ij}^{\top}\bs{1} = \bs{1},\qquad 1\leq i < j \leq n.
%\label{Eq:Linear:Cons}
\end{equation}
\end{thm}
\begin{proof}See the supplemental material.
\end{proof}

Intuitively, this property arises from the following features of $\boldsymbol{x}$ and $\bs{X}_{ii}$:
\begin{align*}
\bs{x}_i^\top\cdot \bs{1} = 1,\quad \bs{X}_{ii}\bs{1} = \bs{x}_i, \ \bs{X}_{ii}^{\top}\bs{1} = \bs{1}, \quad 1\leq i \leq n.
\end{align*}
These intrinsic properties are then {\em propagated} to all off-diagonal blocks by the semidefinite constraint.

\subsubsection{\bf Invariance under variable reparameterization. }
Pioneered by the beautiful relaxation proposed for the MAX-CUT problem \cite{goemans1995improved}, many SDP approaches developed for combinatorial problems employ the integer indicator $\bs{y} = \frac{1}{2}(\bs{1} + \bs{x})$ to parameterize binary variables (e.g.,~\cite{Torr:2003:SDP,kumar2009analysis}). If one applies matrix lifting $\bs{Y} = \bs{y}\bs{y}^{\top}$ and follows a similar relaxation procedure, the resulting semidefinite relaxation (referred to as SDR2) can be derived as follows
%When exploring semidefinite / quadratic relaxation for MAP estimation, several works (e.g.,~\cite{Torr:2003:SDP,kumar2009analysis}) propose to utilize integer indicator $\bs{y} = \frac{\bs{1} + \bs{x}}{2}$ to parameterize the states of each discrete variable.

\begin{eqnarray}
\small\underset{\boldsymbol{y},\boldsymbol{Y}}{\text{maximize}} & \ \sum\limits _{i=1}^{n}\left\langle \boldsymbol{\overline{w}}_{i},\boldsymbol{y}_{i}\right\rangle +\frac{1}{2}\sum\limits _{(i,j)\in\mathcal{G}}\left\langle \boldsymbol{W}_{ij},\boldsymbol{Y}_{ij}\right\rangle \nonumber\\
\small\text{subject to}& \ \footnotesize \left(\begin{array}{cc}
1 & \boldsymbol{y}^{\top}\nonumber\\
\boldsymbol{y} & \boldsymbol{Y}
\end{array}\right)\succeq{\bf 0},\nonumber\\
 &  \ \footnotesize{\bf 1}^{\top}\boldsymbol{y}_{i}=2-m,\qquad\quad\text{ }\text{ }1\leq i\leq n,\nonumber\\
 &  \ \footnotesize \boldsymbol{Y}_{ij}+\boldsymbol{1}\cdot\boldsymbol{y}_{j}^{\top}+\boldsymbol{y}_{i}\cdot\boldsymbol{1}^{\top}+{\bf 1}\cdot{\bf 1}^{\top}\geq{\bf 0},\nonumber\\
 &  \ \footnotesize \quad\quad\quad\quad\quad\quad\quad\quad\quad\quad\text{ }(i,j)\in\set{G},\nonumber\\
 &  \ \footnotesize\small \frac{{\bf 1}\cdot{\bf 1}^{\top}+\boldsymbol{y}_{i}\cdot\boldsymbol{1}^{\top}+\boldsymbol{1}\cdot\boldsymbol{y}_{i}^{\top}+\boldsymbol{Y}_{ii}}{2}=\text{Diag}(\boldsymbol{1}+\boldsymbol{y}_{i}),\nonumber\\
 &  \ \footnotesize\quad\quad\quad\quad\quad\quad\quad\quad\quad\quad 1\leq i\leq n,\label{Eq:SDR2:EQU}
\end{eqnarray}

%\begin{align}
%\small \underset{\bs{y},\bs{Y}}{\maximize } & \text{ } \sum\limits_{i=1}^{n}\la \overline{\bs{w}}_i, \bs{y}_i \ra + \frac{1}{2}\sum\limits_{(i,j)\in \set{G}}\la \bs{W}_{ij}, \bs{Y}_{ij} \ra \nonumber \\
%\subjectto
%& \quad \left(
%               \begin{array}{cc}
%                 1 & \bs{y}^{\top} \\
%                 \bs{y} & \bs{Y} \\
%               \end{array}
%             \right) \succeq \bs{0},
%\quad \nonumber \\
%                        & \quad  \bs{1}^{\top}\bs{y}_i = 2-m, \qquad\qquad \text{ }\text{ } 1\leq i \leq n, \nonumber \\
%                         & \quad \bs{Y}_{ij} + \bs{1}\bs{y}_j^{\top} + \bs{y}_i\bs{1}^{\top} + \bs{1}\cdot\bs{1}^{\top} \geq \bs{0},  \text{ }\text{ }  (i,j) \in \set{G}, \nonumber \\
%                        & \quad \bs{1}\cdot\bs{1}^{\top}+\bs{y}_i\bs{1}^{\top}+\bs{1}\bs{y}_i^{\top} + \bs{Y}_{ii} \nonumber \\
%                         &\quad\quad = 2\Diag(\bs{1}+\bs{y}_i), \qquad  1\leq i \leq n,  \label{Eq:SDR2:EQU}
%\end{align}
where $\overline{\bs{w}}_i$ are defined as
\[\small\overline{\bs{w}}_i = \bs{w}_i + \frac{1}{2}\left(\sum_{j:(i,j)\in \set{G}}\bs{W}_{ij}\bs{1} + \sum_{j:(j,i)\in \set{G}}\bs{W}_{ji}^{\top}\bs{1}\right).\]

In fact, SDR2 is identical to SDR, as formally stated below.

\begin{thm}\label{invariance} $(\bs{x}^{\star}, \bs{X}^{\star})$ is the solution to SDR if and only if
%$\big(\bs{y} := 2\bs{x}^{\star} - \bs{1}, \bs{Y}^{\star} := 4\bs{X}^{\star} - 2(\bs{x}^{\star}\cdot\bs{1}^{\top}+\bs{1}\cdot\bs{x}^{\star \top})+\bs{1}\cdot\bs{1}^{\top}\big)$
\begin{eqnarray*}
\boldsymbol{y}^{\star} &  & :=2\boldsymbol{x}^{\star}-{\bf 1},\\
\boldsymbol{Y}^{\star} &  & :=4\boldsymbol{X}^{\star}-2\left(\boldsymbol{x}^{\star}\cdot{\bf 1}^{\top}+{\bf 1}\cdot\boldsymbol{x}^{\star\top}\right)+{\bf 1}\cdot{\bf 1}^{\top}
\end{eqnarray*}

is the solution to SDR2.
\end{thm}

\begin{proof} See the supplemental material.
\end{proof}
Despite the theoretical equivalence between SDR2 and SDR, from a numerical perspective, solving SDR2 is much harder than solving SDR. The difficulty arises from the complicated form of the linear constraints enforced by SDR2 (i.e., (\ref{Eq:SDR2:EQU})). Note that the advantage of SDR2 is that all diagonal entries of $\bs{Y}$ are equal to $1$ as follows
$$
\diag(\bs{Y}_{ii}) = 2(\bs{1}+\bs{y}_i) - \bs{1} - \bs{y}_i - \bs{y}_i = \bs{1}, \quad (\ 1\leq i \leq n).
$$
Nevertheless, none of prior SDP algorithms takes full advantage of this property in accelerating the algorithm.

\section{Scalable Optimization Algorithm}
\label{Sec:Optimization}

%In this section, we describe the numerical procedure for solving the MAP problem via SDR. We first describe how to SDR . Then we introduce greedy iterative rounding scheme that results in a solution to the original MAP problem.

The curse of dimensionality poses inevitable numerical challenges when applying general-purpose SDP solvers to solve SDR. Despite their superior accuracy, primal-dual interior point methods (IPM) like SDPT~\cite{toh1999sdpt3} are limited to small-scale problems (e.g.  $nm < 150$ on a regular PC). More scalable solvers such as  CSDP~\cite{helmberg2000spectral} and DSDP~\cite{Benson:2008:ADS} propose to solve the dual problem. However, since the non-negativity constraints $\bs{X}_{ij} \geq \bs{0}$ produce numerous dual variables, these solvers are still far too restrictive for our program --- none of them can solve SDR on a standard PC when $nm$ exceeds 1000.

The limited scalability of interior point methods has inspired a flurry of activity in developing first-order methods, among which the alternating direction method of multipliers (ADMM) \cite{wen2010alternating,boyd2011distributed} proves well suited for large-scale problems. In this section, we propose an efficient variant of ADMM -- referred to as SDPAD-LR (SDP Alternating Direction method for Low Rank structure), which is tailored to the special structure of SDR (including low rank and sparsity) and enables us to solve problems with very large dimensionality.
%Here we summarize the basic steps of SDPAD-LR and our innovations for acceleration, with full details deferred to the supplemental materials.

\subsection{Alternating Direction Augmented Lagrangian Method (ADMM)}
%We summarize below the basic steps of the proposed SDPAD-LR, with details deferred to the supplemental materials.

%We begin by introducing several useful notations.

%Let $\set{A}(\cdot): \R^{N\times M}\rightarrow \R^{K}$ represent a linear operator on matrices, and denote by $\set{A}^{\star}(): \R^{K}\rightarrow \R^{N\times M}$ its conjugate operator satisfying $\langle \set{A}(\bs{X}), \bs{y}\rangle = \langle \bs{X}, \set{A}^{\star}(\bs{y})\rangle$. Denote by $\R^{N\times M}_{+}$ the set of matrices with nonnegative elements, and use $(\cdot)_{+}:\R^{N\times M}\rightarrow \R^{N\times M}_{+}$ to represent the operator that applies $x_{+} = \max(0,x)$ element-wise.  Given a symmetric matrix $\bs{C}$, we use $\bs{C}_{\succeq \bs{0}}$ to represent the projection of $\bs{C}$ onto the positive semidefinite cone. Finally, let $\|\cdot\|_{\set{F}}$ denote the Frobenius norm.

For convenience of presentation,  we denote \[\bs{\overline{X}} := \left(
                                                                                           \begin{array}{cc}
                                                                                             1 & \bs{x}^{\top} \\
                                                                                             \bs{x} & \bs{X} \\
                                                                                           \end{array}
                                                                                         \right),\]
and rewrite SDR in the operator form:
\begin{align}
\textup{minimize}   & \quad \left\langle \bs{C}, \overline{\bs{X}}\right\rangle & \textup{dual variables}\nonumber \\
\textup{subject to} & \quad \set{A}\left(\overline{\bs{X}}\right) = \bs{b},     & \bs{y} \nonumber \\
                    & \quad \set{P}\left(\overline{\bs{X}}\right) \geq \bf{0},             & \bs{z}\geq \bf{0} \nonumber \\
                    & \quad \overline{\bs{X}} \succeq \bf{0},                              & \bs{S}\succeq \bf{0}
\end{align}
where $\bs{C}$ encodes all $\bs{w}_i$ and $\bs{W}_{ij}$, $\set{A}(\overline{\bs{X}}) = \bs{b}$ collects the equality constraints, and $\set{P}(\overline{\bs{X}})$ gathers element-wise non-negative constraints. We let variables $\bs{y}$, $\bs{z}$, and $\bs{S}$ represent the corresponding dual variables for respective constraints. In the sequel, we will start by reviewing SDPAD, i.e., the original alternating direction method introduced in ~\cite{wen2010alternating}, and then present the key modification underlying the proposed efficient variant SDPAD-LR.

\subsubsection{\bf SDPAD: Procedures and Convergence}

SDPAD considers the following augmented Lagrangian:
\begin{align*}
\small\set{L} (\bs{y}, \bs{z}, \bs{S}, \overline{\bs{X}}) & = \la \bs{b}, \bs{y}\ra +\left\langle \set{P}^{\star}(\bs{z}) + \bs{S} - \bs{C} - \set{A}^{\star}(\bs{y}), \overline{\bs{X}} \right\rangle \nonumber \\
& +(2\mu)^{-1} \left\|\set{P}^{\star}(\bs{z}) + \bs{S} - \bs{C} - \set{A}^{\star}(\bs{y})\right\|_{\mathrm{F}}^2,
\end{align*}
where the penalty parameter $\mu$ controls the strength of the quadratic term. As suggested by~\cite{boyd2011distributed}, we initialize $\mu$ with a small value,  and gradually increase it throughout the optimization process.

Let superscript $(k)$ indicate the variable in the $k$th iteration. Each iteration of the SDPAD consists of a dual optimization step, followed by a primal update step given as follows
\begin{equation}
\small\overline{\bs{X}}^{(k)} = \overline{\bs{X}}^{(k-1)} + \frac{\set{P}^{\star}(\bs{z}^{(k)}) + \bs{S}^{(k)} - \bs{C} - \set{A}^{\star}(\bs{y}^{(k)})}{\mu}.
\label{Eq:Update:X}
\end{equation}
Instead of jointly optimizing all dual variables, the key idea of SDPAD is to decouple the dual optimization step into several sub-problems or, more specifically, to optimize $\bs{y},\bs{z},\bs{S}$ in order with other variables fixed. This leads to closed-form solutions for each sub-problem as follows
\begin{align}
\small\bs{y}^{(k)}
%& = \underset{\bs{y}}{\arg\min} \set{L}(\bs{y}, \bs{z}^{(k-1)}, \bs{S}^{(k-1)}, \overline{\bs{X}}^{(k-1)}) \nonumber \\\label{Eq:OPT:y}
 & = (\set{A}\set{A}^{*})^{-1}\Big(\set{A}\big(\bs{S}^{(k-1)}-\bs{C} +\mu \overline{\bs{X}}^{(k-1)}\big)-\mu\bs{b}\Big), \nonumber\\
\small\bs{z}^{(k)}
%& = \underset{\bs{z}\geq 0}{\arg\min} \set{L}(\bs{y}^{(k)}, \bs{z}, \bs{S}^{(k-1)}, \overline{\bs{X}}^{(k-1)}) \nonumber \\ \label{Eq:OPT:z}
& = {\footnotesize \set{P}\left(\bs{C}  - \bs{S}^{(k-1)} - \mu \overline{\bs{X}}^{(k-1)}\right)_{+}}, \nonumber\\
\small\bs{S}^{(k)}
 %& = \underset{\bs{S}\succeq 0}{\arg\min} \set{L}(\bs{y}^{(k)}, \bs{z}^{(k)}, \bs{S},  \overline{\bs{X}}^{(k-1)}) \nonumber \\ \label{Eq:OPT:S}
 &= {\footnotesize \left(\bs{C} +\set{A}^{\star}(\bs{y}^{(k)}) - \set{P}^{\star}(\bs{z}^{(k)}) - \mu \overline{\bs{X}}^{(k-1)}\right)_{\succeq \bs{0}}}. \nonumber
\end{align}
Similar to that considered in ~\cite{wen2010alternating}, our stopping criterion involves measuring of both primal feasibility $\|\set{A}(\overline{\bs{X}}^{(k)}) - \bs{b}\|$ and dual feasibility $\mu(\overline{\bs{X}}^{(k)} - \overline{\bs{X}}^{(k-1)})$.

\para{Convergence property.} In general, convergence properties of SDPAD are known when only equality constraints are present~\cite{wen2010alternating}. However, the inequality constraints of SDR are special in the following two aspects:
\begin{enumerate}[(i)]
\item They are element-wise non-negativity constraints;
\item They are essentially decoupled from other linear constraints.
\end{enumerate}
Property (ii) arises as all equality constraints are concerned with diagonal blocks of $\bs{X}$, while all linear inequality constraints are only enforced on its off-diagonal blocks. Such special structure leads to theoretical convergence guarantees for SDPAD, as stated in the following theorem.

\begin{thm} The SDPAD method presented above converges to the optimizer of SDR.
\label{THM:ADMM}
\end{thm}

\begin{proof} See the supplemental material.
\end{proof}

%Note that Theorem~\ref{THM:ADMM} does not specify the convergence rate of the SDPAD-LR method. The same as the spectral bundle method~\cite{helmberg2000spectral}, its convergence rate is still an open problem.

\subsubsection{\bf SDPAD-LR: Accelerated Method}

%\para{SDPAD-LR for large-scale problems.}
Apparently, the most computationally expensive step of SDPAD is the update of $\bs{S}$, which involves the eigen-decomposition of an  $nm\times nm$ matrix.  This limits the applicability of SDPAD to large-scale problems (e.g. $nm>10^4$). To bypass this numerical bottleneck, we modify SDPAD and present an efficient heuristic called SDPAD-LR, which exploits the low-rank structure of $\overline{\bs{X}}$.

First, we observe that $\bs{S}$ can be alternatively expressed as
$$
\small\bs{S}^{(k)} = \bs{C} + \set{A}^{\star}(\bs{y}^{(k)}) - \set{P}^{\star}(\bs{z}^{(k)}) - \mu \left(\overline{\bs{X}}^{(k)} - \overline{\bs{X}}^{(k-1)}\right).
$$
This allows us to present SDPAD without invoking $\bs{S}$. The detailed steps of SDPAD can now be summarized as in Algorithm~\ref{SDR:ADMM}.

\begin{algorithm}[t]
\caption{SDPAD for solving SDR}
\label{SDR:ADMM}
\begin{algorithmic}
\STATE {\bf input}: $k_{\max} = 1000$, $\epsilon = 10^{-4}$, $\mu_{\min} = 10^{-3}$, $\rho = 1.005$.\\
{\bf initialize}: $\small\overline{\bs{X}}^{(0)} = \overline{\bs{X}}^{(-1)} = \bs{0}$, $\small\bs{y}^{(0)} = 0$, $\small\bs{z}^{(0)} = 0$ \\
%\FOR {$k = 1$ to $k_{\max}$}
\REPEAT
\STATE $\small\overline{\bs{X}}_{\temp}^{(k)} = 2\overline{\bs{X}}^{(k-1)} - \overline{\bs{X}}^{(k-2)}$
\STATE $\small\bs{t}^{(k)}_{\temp} = (\set{A}\set{A}^{\star})^{-1}(\set{A}(\overline{\bs{X}}^{(k)}_{\temp}) - \bs{b})$
\STATE $\small\bs{y}^{(k)} = \bs{y}^{(k-1)} + \mu \bs{t}^{(k)}_{\temp}$
\STATE $\small\bs{z}^{(k)} = \Big(\bs{z}^{(k-1)} - \mu \set{P}(\overline{\bs{X}}_{\temp}^{(k)})\Big)_{+}$
%\STATE
\begin{equation}
\small\overline{\bs{X}}^{(k)} = \Big(\overline{\bs{X}}^{(k-1)} - \frac{\bs{C} + \set{A}^{\star}(\bs{y}^{(k)}) - \set{P}^{\star}(\bs{z}^{(k)})}{\mu}\Big)_{\succeq 0}
\label{Eq:OPT:X}
\end{equation}
\STATE $k\leftarrow k+1;\quad \mu = \mu\rho$
%\IF {{\footnotesize $\min(\mu\|\overline{\bs{X}}^{(k)} - \overline{\bs{X}}^{(k-1)}\|_{\set{F}}, \|\set{A}(\overline{\bs{X}}^{(k)})-\bs{b}\| )\leq \epsilon$}}
%\STATE break;
%\ENDIF
%\ENDFOR
\UNTIL{$\small\min(\mu\|\overline{\bs{X}}^{(k)} - \overline{\bs{X}}^{(k-1)}\|_{\mathrm{F}}, \|\set{A}(\overline{\bs{X}}^{(k)})-\bs{b}\| )\leq \epsilon$ or $k>k_{\max}$}
\end{algorithmic}
\end{algorithm}

It is straightforward to see that the bottleneck of Algorithm~\ref{SDR:ADMM} lies in how to compute and store the primary variable $\overline{\bs{X}}$. To derive an efficient solver, we make the assumption that the optimal solution $\overline{\bs{X}}^{\star}$ is low-rank. This is motivated by the empirical evidence that for a variety of problems (see the experimental section for details), SDR is exact, meaning $\text{rank}(\overline{\bs{X}}^{\star}) = 1$. Moreover, in the general case, the rank of $\overline{\bs{X}}^{\star}$ is expected to be much small than its dimension (e.g. ~\cite{journals/mp/BurerM03}), i.e.,
$$
\text{rank}\left(\overline{\bs{X}}^{\star}\right)\left(\text{rank}(\overline{\bs{X}}^{\star})+1\right) \leq 2M, 
$$
where $M$ is the number of constraints.\footnote{Practically, many negativity constraints are redundant.} of SDPR.

Based on this assumption, the key idea of SDPAD-LR is to invoke a low-rank matrix $\bs{Y} \in \R^{(nm+1)\times  r}$ for some small $r$ and  encode $\overline{\bs{X}} = \bs{Y}\bs{Y}^{\top}$ throughout the iterative process. This allows us to keep all the variables in memory even for large-scale problems.

In this case, (\ref{Eq:OPT:X}) is modified as $\bs{Y}^{(k)} = \bs{U}^{(k)}{\bf\Sigma}_{+}^{\frac{1}{2}}$, where ${\bf\Sigma} = \Diag(\sigma_1, \cdots, \sigma_r)$ and $\bs{U} = (\bs{u}_1, \cdots, \bs{u}_{r})$ represent the top $r$ eigenvalues and respective eigenvectors of
\begin{equation}
\small\bs{V}^{(k)} =  \bs{Y}^{(k-1)}{\bs{Y}^{(k-1)\top}} - \frac{\bs{C} + \set{A}^{\star}(\bs{y}^{(k)}) - \set{P}^{\star}(\bs{z}^{(k)})}{\mu}.\label{variant}
\end{equation}
Although $\bs{V}^{(k)}$ is a dense matrix, its top eigenvectors can be efficiently computed using the {\em Lanczos process} \cite{cullum2002lanczos}, whose efficiency is dictated by the complexity of the matrix multiplication operator $\bs{V}^{(k)}: \bs{u} \in \R^{nm+1} \rightarrow \bs{V}^{(k)}\bs{u}\in \R^{nm+1}$. As SDR only involves the constrains $\bs{X}_{ij}\geq \bs{0}, (i,j)\in \set{E}$, the matrix $\bs{C} + \set{A}^{\star}(\bs{y}^{(k)}) - \set{P}^{\star}(\bs{z}^{(k)})$ turns our to share the same sparsity pattern with $\set{G}$. Thus, the complexity of computing $\bs{V}^{(k)}\bs{u}$ is at most $O(nmr^2 + m^2 |\set{E}|)$.

%Following a similar proof of Theorem~\ref{THM:ADMM}, it can be shown that the modified algorithm, with (\ref{Eq:OPT:X}) replaced by (\ref{variant}), still converges (here we assume the number of iterations is sufficiently large). Although there is no guarantee that the modified algorithm converges to an optimal solution to SDR, it is easy to see that if $r$ is larger than the number of positive eigenvalues of $\bs{V}^{(k)}$, then the modified algorithm indeed returns an optimal solution. Practically, our experiments reveal that the modified algorithm usually converges to the global optimal even for very small $r$.

\begin{algorithm}[t]
\caption{SDPAD-LR for solving SDR}
\label{SDR:ADMM1}
\begin{algorithmic}
\STATE {\bf input}: $k_{\max} = 5000$, $\epsilon = 10^{-4}$, $\mu_{\min} = 10^{-3}$, $\rho = 1.005$, $\delta = 1e-2$, $r_{\max} = 32$, $r = 4$.\\
{\bf initialize}: $\small\overline{\bs{X}}^{(0)} = \overline{\bs{X}}^{(-1)} = \bs{0}$, $\small\bs{y}^{(0)} = 0$, $\small\bs{z}^{(0)} = 0$ \\
%\FOR {$k = 1$ to $k_{\max}$}
\REPEAT
\STATE $\small\overline{\bs{X}}_{\temp}^{(k)} = 2\overline{\bs{X}}^{(k-1)} - \overline{\bs{X}}^{(k-2)}$
\STATE $\small\bs{t}^{(k)}_{\temp} = (\set{A}\set{A}^{\star})^{-1}(\set{A}(\overline{\bs{X}}^{(k)}_{\temp}) - \bs{b})$
\STATE $\small\bs{y}^{(k)} = \bs{y}^{(k-1)} + \mu \bs{t}^{(k)}_{\temp}$
\STATE $\small\bs{z}^{(k)} = \Big(\bs{z}^{(k-1)} - \mu \set{P}(\overline{\bs{X}}_{\temp}^{(k)})\Big)_{+}$
\STATE Compute $\small\overline{\bs{X}}^{(k)}$ according to (\ref{variant})
\STATE $k\leftarrow k+1; \quad \mu = \rho \mu$
%\IF {{\footnotesize $\min(\mu\|\overline{\bs{X}}^{(k)} - \overline{\bs{X}}^{(k-1)}\|_{\set{F}}, \|\set{A}(\overline{\bs{X}}^{(k)})-\bs{b}\| )\leq \epsilon$}}
%\STATE break;
%\ENDIF
%\ENDFOR
\IF {\footnotesize $\mod(k, 1000) = 0$, $\small\lambda_{\min}(\overline{\bs{X}}^{(k)} ) > \delta \lambda_{\max}(\overline{\bs{X}}^{(k)})$}
\STATE $r = \min(r_{\max}, 2r); \quad \mu = \mu_{\min}$
\ENDIF
\UNTIL{$k>k_{\max}$ or $\small\lambda_{\min}(\overline{\bs{X}}^{(k)})\leq \delta\lambda_{\max}(\overline{\bs{X}}^{(k)})$ and $\small\min(\mu\|\overline{\bs{X}}^{(k)} - \overline{\bs{X}}^{(k-1)}\|_{\mathrm{F}}, \|\set{A}(\overline{\bs{X}}^{(k)})-\bs{b}\| )\leq \epsilon$}
\end{algorithmic}
\end{algorithm}

Theoretically, it is extremely challenging to derive an upper bound on $r$ to ensure the exactness of the modified algorithm. To address this issue, we thus design SDPAD-LR so that it iteratively doubles the value of $r$ and reapplies the modified algorithm until it returns the optimal solution. For most of our experiments, we found that $r = 8$ is sufficient. 

The pseudo-code of SDPAD-LR is summarized in Algorithm~\ref{SDR:ADMM1}.

%On the other hand, it is extremely to derive a upper bound on $r$ so that this property holds.
%One can interpret $\bs{X}^{(k)} = \bs{Y}^{(k)} {\bs{Y}^{(k)}}^{T}$ as the rank-$r$ approximation of $\bs{V}^{(k)}$ in the positive semidefinite cone. Clearly, if $r$ is larger than the number of positive eigenvalues of $\bs{V}^{(k)}$, then (\ref{variant}) is identical to (\ref{Eq:OPT:X}). Interestingly, our experiments reveal that SDPAD-LR converges even for very small $r$ ($r = 6$ for all the experiments herein).

%Consequently, this property enables us to apply SDPAD-LR with remarkable scalability, particularly for sparse problems (e.g., the ones from OPENGM2).

\subsection{Iterative Rounding}

Similar to other ADMM methods~\cite{boyd2011distributed}, SDPAD-LR converges rapidly to moderate accuracy within the first 400 iterations, and significantly slows down afterwards. Thus, rather than continuing until SDPAD-LR converges, it would be more efficient to shrink the problem size by fixing those variables whose optimal states are likely to have been revealed. Specifically, after each round of SDPAD-LR, we fix the optimal state $j$ of a variable $x_i$ if $x_{i,j} > t_{\max}$ ($t_{\max} = 0.99$ for all the examples) or $x_{i,j} = \max_{1\leq i \leq n, 1\leq j \leq m}x_{i,j}$. We then reapply the iterative procedures on the reduced problem. In practice, we find that due to the tightness of SDR, the size of the reduced problems are significantly smaller than the original problem, and one iterative rounding procedure is usually sufficient.

\section{Experimental Results}
\label{Sec:Results}

\begin{table*}[t]
  \caption{Comparison of SDP Solvers on Representative Problems. $N$: dimension of the matrix. $M$: number of constraints.}
\begin{center}
\scriptsize
\begin{tabular}{ccccccccccccc}
\hline
\multirow{3}{*}{\footnotesize Method}& \multicolumn{3}{c}{\footnotesize  deer\_0034.K10.F100 (dense)}  &  \multicolumn{3}{c}{\footnotesize  file\_30markers (sparse)}  & \multicolumn{3}{c}{\footnotesize  folding\_2BE6 (dense)}   &\multicolumn{3}{c}{\footnotesize  gm275 (sparse)} \\ & \multicolumn{3}{c}{\footnotesize  $N = 661, M = 218791$}  &  \multicolumn{3}{c}{\footnotesize  $N = 862, M = 218791$}  & \multicolumn{3}{c}{\footnotesize  $N = 3836, M = 218791$}   &\multicolumn{3}{c}{\footnotesize $N = 5201, M = 218791$} \\ \cline{2-13}
& cpu & gap & inf & cpu & gap & inf & cpu & gap & inf & cpu & gap & inf \\ \hline
{\footnotesize SDPAD-LR} & 4:33 & 7.2e-4 & 1.3e-6 & 7:33 & 2.2e-4 &5.3e-6 & 2:44:36 & 2.3e-4 & 5.3e-7 & 21:33 & 5.1e-4 & 1.3e-6 \\
{\footnotesize SDPAD} & 8:29 & 8.2e-5 & 4.3e-7 & 10:33 & 9.4e-5 & 1.3e-7 & 25:56:37 & 2.3e-4 & 3.7e-6 & 41:33:21 & 1.2e-4 & 3.1e-6  \\
{\footnotesize SDPNAL} & 10:55 & 8.1e-5 & 1.3e-6 & 9:42 & 6.2e-5 & 2.1e-6 & 18:33:11 & 5.2e-5 & 4.7e-7 & 21:34:35 & 9.7e-5 & 4.5e-7 \\
{\footnotesize IPM-NC}& 1:27 & 2.3e3 & na & 2:37 & 4.1e-7 & na & 10:23 & 4.5e2 & na & 21:56 & 3.5e-6 & na \\
{\footnotesize MOSEK }& 21:33:10 & 2.3e-6 & 1.3e-9 & \multicolumn{3}{c}{\scriptsize na} &\multicolumn{3}{c}{\scriptsize na}& \multicolumn{3}{c}{\scriptsize na} \\
{\footnotesize MUL-Update}& 6:13:56 & 8.1e-3 & 2.7e-5 & \multicolumn{3}{c}{\scriptsize na}&\multicolumn{3}{c}{\scriptsize na} & \multicolumn{3}{c}{\scriptsize na}\\ \hline
\end{tabular}
\end{center}
\normalsize
\label{Table:SDP:Comparison}
\end{table*}

In this section, we evaluate SDPAD-LR on several benchmark data sets and compare its performance against existing SDP solvers and state-of-the-art MAP inference algorithms.

\subsection{Benchmark Datasets}

%We first give a brief description of the benchmark data sets that are used for evaluation. Then we introduce the baseline algorithms to be compared as well as the evaluation protocol. Due to space constraints, we only include representative categories from each benchmark  and top-performing baseline algorithms in the main paper. Please refer to the supplemental material for complete results.

\begin{table}[h]
\scriptsize
\begin{center}
\begin{tabular}{ lccccc}
\hline
{\footnotesize  categories}       & {\footnotesize $\set{G}$ } & {\footnotesize $n$}               &{\footnotesize $m$}           & {\footnotesize probs}  & {\footnotesize $t$}   \\ \hline
{\footnotesize  PIC-Object  }     & full       & 60                & 11-21       & 37     & 5m32s\\
{\footnotesize  PIC-Folding }      & mixed      & 2K               & 2-503        & 21     & 21m42s\\
{\footnotesize  PIC-Align  }  & dense      & 30-400            & 20-93           & 19     & 37m63s\\ \hline
{\footnotesize  GM-Label }        & sparse     & 1K               & 7            & 324       & 6m32s\\
 {\footnotesize GM-Char }    & sparse     & 5K-18K            & 2                & 100       & 1h13m\\
{\footnotesize  GM-Montage }      & grid       & 100K              & 5,7         & 3       & 9h32m\\
{\footnotesize  GM-Matching}      & dense      & 19                & 19          & 4         & 2m21s\\ \hline
{\footnotesize  ORIENT}      & sparse      & 1K                & 16          & 10         & 10m21s\\ \hline
%{\footnotesize  ORIENT }     & dense     & 500                & 16         & 3        & 25m32s\\ \hline
\end{tabular}
\end{center}
\normalsize
  \caption{Statistics of the datasets evaluated in this paper. $\set{G}$: graph structure of the MAP problem in each category; $n$: number of variables; $m$: number of states; probs: number of instances; $t$: average running time of SDPAD-LR.}
\label{Table:Stats}
\end{table}

We perform experimental evaluation on MAP estimation problems from three popular benchmark data sets (See Table~\ref{Table:Stats}), i.e., OPENGM2~\cite{Kappes:2013:OPENGM2}, PIC~\cite{PIC:2011}, and a new data set ORIENT for the task of estimating consistent camera orientations~\cite{crandall2011discrete}. OPENGM2 comprises 19 categories of mostly sparse MAP problems.  We choose four representative categories for evaluation: Geometric Surface Labeling (GM-Label), Chinese Characters (GM-Char), MRF Photomontage (GM-Montage) and Matching (GM-Matching). The first three categories GM-Label, GM-Character and GM-Montage are sparse MAP estimation problems with increasing scales. GM-Matching is a special category where our convex relaxation is not tight. PIC comprises 10 categories of MAP inference problems of various structure. As we already include sparse MAP inference problems from OPENGM2, we pick 3 representative dense categories from PIC: Object Detection(PIC-Object), Image Alignment (PIC-Align) and Folding (PIC-Folding).
%Finally, ORIENT includes three instances from~\cite{hsg-fgssl-13}. Each instance consists of 500 shapes, where each shape has $16$ candidate orientations and is connected with 50 neighboring shapes.

\subsection{SDP Solver Evaluation}

\para{Baseline algorithms.} We evaluate the proposed SDPAD-LR against the following existing large-scale SDP solvers.
\begin{itemize}
\itemsep0.2em
\item SDPAD --- the original ADMM method presented in~\cite{wen2010alternating}.
\item SDPNAL --- the Newton-CG (conjugate gradient) augmented method proposed in~\cite{Zhao:2010:NAL}.
\item IPM-NC --- the nonconvex interior point method which attempts to solve a direct relaxation of the MAP inference problem~\cite{journals/mp/BurerM03}:
    \begin{align*}
    \minimize & \quad \langle \bs{C}, \bs{x}\bs{x}^{\top}\rangle  \\
    \subjectto & \quad \bs{1}^{\top}\bs{x}_i = 1, \bs{x}_i \geq 0, \quad 1\leq i \leq n
    \end{align*}
This method serves as an alternative low-rank heuristic for the proposed SDPAD-LR. With losing generality, we set the initial values of $\bs{x}_i = \frac{1}{m}\bs{1}, 1\leq i \leq n$.
\item MOSEK --- the cutting-edge interior point method. To apply it on large-scale SDRs, we add the nonnegativity constraints in an incremental fashion, i.e., at each iteration, we detect the 100 smallest negative entries and add them to the constraint set.
\item MUL-Update --- an approximate on-line SDP solver that is based on multivariate weight updates~\cite{journals/toc/AroraHK12}.
\end{itemize}

\para{Problem sets.} For evaluation, we consider four categories, on which most baseline algorithms are applicable: PIC-OBJ, PIC-Align, PIC-Folding and GM2-Label. For simplicity, we pick a representative problem from each category. The dimensions of these problem sets range from $600$ to $5000$, and they contain both dense and sparse problems (See Table~\ref{Table:SDP:Comparison}).

\begin{table*}[t]
  \caption{Results on benchmark datasets.}
\begin{center}
\footnotesize
\begin{tabular}{ccccccccc}
\hline
& {\footnotesize SDPAD-LR}    &{\footnotesize Ficolofo }    &{\footnotesize BRAOBB } &{\footnotesize $\alpha$-expand}      &{\footnotesize TRWS-LF2 }  & {\footnotesize ogm-TRBP}  &{\footnotesize MCBC }  & {\footnotesize A-star}   \\ \hline\hline
\scriptsize
  \multirow{2}{*}{\footnotesize ORIENT}
  & \textsl{\textbf{-7834.6}}                                 &     \multirow{2}{*}{na}                               & \textsl{-3059.2}                                 &     \textsl{-7695.4}                                                     & \textsl{-7592.4} & \textsl{-7553.8}
  & \multirow{2}{*}{na} & \multirow{2}{*}{na}\\
  &     \textbf{100\%}                             &                           & 0\%                           &            0\% & 0\% & 0\%
  &   &                        \\\hline

\multirow{2}{*}{\footnotesize PIC-Object} &\textsl{\textbf{-19316.12}} & \textsl{-19308.94} & \textsl{-19113.87} & \textsl{-10106.8} & \textsl{-19020.82} & \textsl{-18900.81}
&\multirow{2}{*}{na} & \multirow{2}{*}{na}\\
&\textbf{97.3\%}& 91.9\%& 24.3\%& 0\%& 59.5\%& 32.2\%
& & \\ \hline
\multirow{2}{*}{\footnotesize PIC-Folding} &\textsl{\textbf{-5963.68}} & \textsl{\textbf{-5963.68}}& \textsl{-5927.01}&\textsl{-5652.76} &\textsl{-5905.01} &\textsl{-5907.24}
&\multirow{2}{*}{na} & \multirow{2}{*}{na}\\
& \textbf{100\%} & \textbf{100\%}& 42.9\% & 14.2\%&38.1\% &42.9\% & & \\ \hline
\multirow{2}{*}{\footnotesize PIC-Align} & \textsl{\textbf{2285.23}} & \textsl{2285.34} & \textsl{2285.34} & \textsl{2285.34}& \textsl{2286.64}&\textsl{2289.12}
&\multirow{2}{*}{na} &\multirow{2}{*}{na} \\
& \textbf{100\%} & 90\% & 90\% &90\% & 80\%& 70\%& & \\\hline
\multirow{2}{*}{\footnotesize GM-Label} & \textsl{\textbf{-476.95}} & \multirow{2}{*}{na}&\multirow{2}{*}{na} &\textsl{\textbf{-476.95}} &\textsl{-476.95} & \textsl{486.42}
&\multirow{2}{*}{na} & \multirow{2}{*}{na}\\
& \textbf{100\%} & & &\textbf{100\%} &99.67\% &40\% & & \\\hline
\multirow{2}{*}{\footnotesize GM-Char} & \textsl{\textbf{-59550.67}} &\multirow{2}{*}{na} & \multirow{2}{*}{na}&\multirow{2}{*}{na} & -49519.44&-49507.98	
&\textsl{-49550.10} & \multirow{2}{*}{na}\\
& 86.1\% & & & & 11\% & 6\% & \textbf{89.1\%}& \\\hline
\multirow{2}{*}{\footnotesize GM-Montage} & \textsl{168298.00} &\multirow{2}{*}{na} &\multirow{2}{*}{na} & \textsl{\textbf{168220.00}} & \textsl{735193.0} & \textsl{235611.00}
&\multirow{2}{*}{na} &\multirow{2}{*}{na} \\
& \textbf{66.3\%} & & & 33.3\% &0\% &0\% & & \\\hline
\multirow{2}{*}{\footnotesize GM-Matching} & \textsl{44.19} & \multirow{2}{*}{na} & \textsl{\textbf{21.22}} & \multirow{2}{*}{na} & \textsl{32.38} & \textsl{5.5e10}
&\multirow{2}{*}{na} & \textsl{\textbf{21.22}}\\
& 0\% & &\textbf{100\%} & & 0\% & 0\%& & \textbf{100\%} \\\hline
\end{tabular}
\end{center}
\normalsize
\label{Table:ALL}
\end{table*}

\para{Evaluation protocol.} Following the standard protocol for assessing convex programs, we evaluate the duality gap and the primal/dual infeasibility of each algorithm:
\begin{align*}
\textup{gap} & = \frac{|\langle \bs{b}, \bs{y} \rangle-\langle \bs{C}, \bs{X} \rangle|}{1+|\langle \bs{b}, \bs{y} \rangle|+|\langle \bs{C}, \bs{X} \rangle|}, \\
\textup{inf} & = \max\Big\{\frac{\|\set{A}(\bs{X}) - \bs{b}\|_2+\|\min(\set{P}(\bs{X}),0)\|_2}{1+\|\bs{b}\|_2},\\
&\qquad\qquad \frac{\|\bs{C}+\set{A}^{*}(\bs{y}) - \set{P}^{*}(\bs{z})-\bs{S}\|_\mathrm{F}}{1+\|\bs{C}\|_\mathrm{F}}\Big\}
\normalsize
\end{align*}
As IPM-NC solves a different optimization problem, we report the gap between its optimal solutions with the ground-truth optimal solutions.

\para{Analysis of results.} We run each algorithm until the duality gap is below $1e-4$ or the maximum number of iterations is reached. Table~\ref{Table:SDP:Comparison} shows the running time, duality gap and maximum primal/dual infeasibility of each algorithm on each problem. We can see that SDPAD-LR generates results that are comparable to SDPAD and SDPNAL. However, SDPAD-LR turns out to be remarkably more efficient than SDPAD and SDPNAL on large-scale or sparse datasets. This is due to the fact that SDPAD-LR only requires computing the top eigenvalues, which is both memory and computationally efficient.

Both interior point methods (i.e., IPM-NC and MOSEK) have provable guarantees to generate more accurate results than other methods. However, MOSEK is not scalable to large data sets, as reported in Table~\ref{Table:SDP:Comparison}. IPM-NC is scalable to large-scale problems, as the number variables involved is small. However, as IPM-NC solves a non-convex optimization problem, it may easily get trapped into local minimals (e.g., on deer\_0034.K10.F100\_30markers and folding\_2BE6).

Finally, the multivariate weight update method MUL-Update turns out be inefficient on solving SDRs of MAP inference problems. This is due to the fact that MUL-Update is an approximate solver and it requires a lot of iterations to obtain an accurate solution.

\subsection{MAP Inference Evaluation}

\para{Experimental setup.} We compare SDR with the top-performing algorithms from OPENGM2~\cite{Kappes:2013:OPENGM2}. These algorithms include (i) BRAOBB~\cite{journals/aicom/OttenD12}, which is based on combinatorial search, (ii) $\alpha$-expansion~\cite{Szeliski:2008:CSE}--a move making method, (iii) MCBC~\cite{Kappes-2013}, which is based on a highly optimized max-cut solver, (iv) TRWS-LF2~\cite{Kolmogorov:2006:TRM}-- Tree-reweighted message passing, (vi) ogm-TRBP--- Tree-reweighted belief propagation~\cite{Szeliski:2008:CSE} and (vii) ficolofo~\cite{journals/ai/CooperGSSZW10}-- the top performing method on dense problems of PIC.

We use two measures to assess the performance of each method. The first measure evaluates for each method the mean objective values $\overline{f}$ of the resulting MAP assignments on each category. For the consistency with~\cite{Kappes:2013:OPENGM2}, we report $-\overline{f}$, meaning that the smaller the value, the better the algorithm. The second measure reports the percentage that each method achieves the best solution among all existing methods (not necessarily the global optimal). The higher the percentage, the better the algorithm.

\para{Performance.} Table~\ref{Table:ALL} summarizes the performance of SDPAD-LR v.s. state-of-the-art MAP inference algorithms on each type of problems. In each block, the top element (which is tilted) describes $-\overline{f}$ of each method on each category, and the bottom block describes the percentage of obtaining the best solution. We can see that the overall performance of SDPAD-LR is superior to each other individual algorithm. Except on GM-Matching, SDPAD-LR is the top performing on each other dataset. In contrast, each existing method either does not apply or generates poor results on one or several datasets. This shows the advantage of solving a strong convex relaxation of the MAP inference problem. Below we break down the performance on each benchmark.

\begin{itemize}
\itemsep0em
\item
\para{ORIENT.} SDPAD-LR is the leading method on ORIENT. The problems in ORIENT exhibit specific structures, i.e, the pair-wise potentials consist of approximately shifted permutation matrices. Experimentally, we found that SDR is usually tight on these problems. This explains the superior performance SDPAD-LR. In contrast, linear programming relaxations are not tight on ORIENT, and thus TRBP and TRWS only deliver moderate performance. Moreover, this structural pattern leads to huge search spaces for combinatorial algorithms (e.g., BRAOBB), and they can easily get stuck in local optimums.

\item\para{Dense problems.} SDPAD-LR also outperforms other methods on three dense categories from PIC. It achieves the best mean energy value as well as the highest percentage of obtaining the best solution. This again arises since SDR is tight on these problems.

\item\para{Sparse problems.} SDR yields comparable results with state-of-the-art algorithms on the three sparse categories from OPENGM2. GM-Label consists of problems where the standard LP relaxation is tight. On GM-Char which consists of large-scale binary problems, SDR is comparable to MCBC in the sense that SDR achieves a better mean energy value while MCBC attains a higher percentage of being the best solution. This arises because MCBC is a highly optimized solver designed for binary quadratic problems. On the other hand, SDPAD-LR is only an approximate SDP solver which, in some cases, may not converge to the global optimum due to numerical issues.

\item\para{GM-Matching.} SDR only yields moderate results on GM-Matching. This occurs because SDR is not tight on GM-Matching. In contrast, as GM-Matching is a small-scale problem, combinatorial optimization techniques such as BRAOBB and A-star are capable of finding globally optimal solutions.

\end{itemize}

\para{Running Times.} The running time of SDPAD-LR (including the rounding procedure) is of the same scale as other convex relation techniques. As shown in Table~\ref{Table:Stats}, our preliminary Matlab implementation takes less than 10 mins on small-scale problems (i.e. those in PIC-Object, GM-Matching and PIC-Label). On medium size problems, i.e., those in PIC-Folding, PIC-Align, GM-Char and ORIENT, the running time of SDPAD-LR ranges from 20 minutes to 1 hour. On large-scale problems from GM-Montage, SDPAD-LR takes around 8 hours on each problem. However, there is still huge room for improvement. One alternative is to use the eigenvalues computed in the previous iteration to accelerate the eigen-decomposition at the current iteration, which is left for future work.

%\subsection{Analysis of SDR and ADALM-SDP}

%\para{Tightness of SDR.} Encouragingly, SDR is pretty tight on a wide spectrum of problems. We have applied the MOSEK in CVX --- an accurate SDP solver, to solve SDR on PIC-Object. Numerical results indicate that SDR is exact on PIC-Object.
%For large-scale problems, we have applied the exactness solution conditions described in Section~\ref{Sec:Analysis} to check the global optimality of the solutions obtained from ADALM-SDP. It turns out that the majority of these solutions are guaranteed to be globally optimal.
%This highlights the power of SDR on real examples.
%as well as the practicability of the exact solution conditions.

%\para{Efficiency of ADALM-SDP.} ADALM-SDP is an approximate solver and thus may produce fractional solutions even when the solution to SDR is exact. However, we observe that when combined with appropriate iterative rounding procedures, ADALM-SDP generates a solution that is very similar to the one returned by accurate SDP solvers. For example, on PIC-Object, ADALM-SDP yields global optimum for 36 out of 37 instances.

\section{Conclusions}
\label{Sec:Conclusions}

In this paper, we have presented a novel semidefinite relaxation for second-order MAP estimation and proposed an efficient ADMM solver. We have extensively compared the proposed SDP solver with various state-of-the-art SDP solvers. Experimental results confirm that our SDP solver is much more scalable than prior approaches when applied to various MAP estimation problem, which enables us to apply SDR on large-scale datasets. Owing to the power of semidefinite relaxation, SDR proves superior to other top-performing MAP inference algorithms on a variety of benchmark datasets.

There are plenty of opportunities for future research. First, we would like to extend SDR to higher-order MAP problems. Moreover, it would be interesting to integrate SDR and combinatorial optimization techniques, which has the potential to boost the power of both. From the theoretical side, theoretical support for exact estimation with SDR would be one exciting direction for investigation. This would offer justification of the presented low-rank heuristic. On the other hand, as many combinatorial optimization problems can be formulated as MAP inference problems, such exact estimation conditions can shed light on the original combinatorial optimization problems.

\section*{Acknowledgments}
This work has been supported in part by NSF grants FODAVA 808515 and CCF 1011228, AFOSR grant FA9550-12-1-0372, ONR MURI N00014-13-1-0341, and a Google research award.

\appendix

% In the unusual situation where you want a paper to appear in the
% references without citing it in the main text, use \nocite

\appendix

\bibliography{matching}
\bibliographystyle{icml2014}

\end{document}